\documentclass[letterpaper, 10 pt, conference]{ieeeconf}  % Comment this line out if you need a4paper

\IEEEoverridecommandlockouts                              % This command is only needed if 
\usepackage{amsmath, amssymb, amsfonts}       % blackboard math symbols
\usepackage{balance} 
\usepackage{xcolor}
\usepackage{acro}
\usepackage{graphicx}
\usepackage{subcaption}
\usepackage{siunitx}
\usepackage{booktabs}
\usepackage{bm}
\usepackage{cite}

\newtheorem{proposition}{Proposition}[section]
\newtheorem{definition}{Definition}[section]

\renewcommand{\thecorollary}

\newcommand{\reffig}[1]{Fig.~\ref{#1}}
\newcommand{\reftab}[1]{Table~\ref{#1}}
\newcommand{\refsec}[1]{Section~\ref{#1}}

\newcommand{\refequ}[1]{Eq.~\eqref{#1}}

\DeclareAcronym{sUAS}{
  short = sUAS,
  long  = small uncrewed aerial system,
  short-indefinite = an,
  long-indefinite = a
}

\DeclareAcronym{UAV}{
  short = UAV,
  long  = uncrewed aerial vehicle,
  short-indefinite = a,
  long-indefinite = a
}

\DeclareAcronym{DEM}{
  short = DEM,
  long  = digital elevation map
}

\DeclareAcronym{ICS}{
  short = ICS,
  long  = inevitable collision state,
  short-indefinite = an,
  long-indefinite = an
}

\title{\LARGE \bf
Safe Periodic Trochoidal Paths for Fixed-Wing UAVs\\in Confined Windy Environments
}

\author{Jaeyoung Lim$^{1}$, David Rohr$^{1}$, Thomas Stastny$^{1}$, Roland Siegwart$^{1}$% <-this % stops a space
\thanks{$^1$Autonomous Systems Lab, ETH Z\"urich, Switzerland. {\tt \footnotesize \{jalim, drohr, tstastny, rsiegwart\}@ethz.ch}}%
\thanks{This work was supported by InnoSuisse Innovation Project 17.867 IP-ICT in collaboration with Dufour Aerospace AG.}%
}

\begin{document}

\maketitle
\thispagestyle{empty}
\pagestyle{empty}

%%%%%%%%%%%%%%%%%%%%%%%%%%%%%%%%%%%%%%%%%%%%%%%%%%%%%%%%%%%%%%%%%%%%%%%%%%%%%%%%
\begin{abstract}
Due to their energy-efficient flight characteristics, fixed-wing type \acp{UAV} are useful robotic tools for long-range and duration flight applications in large-scale environments.
However, flying fixed-wing \ac{UAV} in confined environments, such as mountainous regions, can be challenging due to their limited maneuverability and sensitivity to uncertain wind conditions. 
% In this work, we first analyze periodic trochoidal paths that can be used to define a safe terminal set, and second propose a wind-invariant safe set along with a switching strategy for selecting the corresponding minimum-extent periodic path type. 
%
In this work, we first analyze periodic trochoidal paths that can be used to define wind-aware terminal loitering states.
We then propose a wind-invariant safe set of trochoidal paths along with a switching strategy for selecting the corresponding minimum-extent periodic path type.
Finally, we show that planning with this minimum-extent set allows us to safely reach up to 10 times more locations in mountainous terrain compared to planning with a single, conservative loitering maneuver.

\end{abstract}

%%%%%%%%%%%%%%%%%%%%%%%%%%%%%%%%%%%%%%%%%%%%%%%%%%%%%%%%%%%%%%%%%%%%%%%%%%%%%%%%
\section{INTRODUCTION}
% \begin{itemize}
%     \color{blue}
%     \item Why UAV are useful platforms.
%     \item Use of periodic paths and implication to safety.
%     \item Safe invariant sets and consideration to wind uncertainty.
%     \item Dubins and Trochoids.
%     \item Contributions.
% \end{itemize}

% Why we care about fixed-wing aerial vehicles.
\Acfp{UAV} have become crucial tools for information-gathering applications, such as surveying and inspection~\cite{bircher2016threedimensional}, search and rescue~\cite{oettershagen2018robotic}, and environment monitoring~\cite{jouvet2019high, buhler2017photogrammetric}. For large-scale coverage or long-range applications, fixed-wing type \acp{UAV} are preferred over rotary-wing type systems due to their high endurance and speed. While the wing-borne aerodynamic lift enables energy-efficient flight, it also poses challenges for operating safely. Fixed-wing \ac{UAV} need to maintain a minimum air-relative velocity (airspeed), making them non-holonomic and coupled with the wind. This can become especially critical in, e.g., mountainous terrain where the inability to stop and the high uncertainty associated with wind greatly increases the risk of collisions. 
% Therefore, fixed-wing \ac{UAV} operations to date have mostly been limited to large open spaces.

% Use of periodic paths to approximate safety
%   Limitation: 
%     - Dubins paths are Considered without wind
Planning safe paths for fixed-wing vehicles in confined environments has been demonstrated in tight indoor spaces~\cite{bry2015aggressive} and steep, mountainous terrain~\cite{oettershagen2017towards, lim2024safe}.
An important property of evaluating safe planning is to ensure infinite-horizon collision checks.
\cite{lim2024safe} approximates infinite horizon collision checks by checking whether a circular path exists that satisfies the terrain (or airspace) constraints. Circular paths are natural periodic paths in zero wind conditions, representing a fixed-wing \ac{UAV} maintaining constant speed and roll angle. As the path is periodic, evaluation of the path for a single period can be applied repeatedly for the consecutive periods. This is then used as an efficient way to evaluate the safety of the terminal state for the path planning problem. However, in wind, a loitering fixed-wing \ac{UAV} must continuously adapt its roll angle to stay on track. Depending on the magnitude of the wind and radius of the circle, tracking may even become infeasible.

% - Trochoids have been considered before, but uncertainty considers worst case which results in extremely conservative bounds
Wind-aware planning for fixed-wing vehicle has been considered for uniform~\cite{mcgee2005optimal, mcgee2007optimal, moon2023time, techy2009minimum, techy2010planar, bucher2023robust} and non-uniform wind cases~\cite{oettershagen2017towards, schopferer2018path, duan2024energy}. While these methods can correctly consider wind with air-relative curvature constraints, they do not evaluate the safety in the planning problem. 
% In order to navigate safely in confined environments under the influence of wind, it is important to be able to evaluate whether the vehicle would enter \iac{ICS}. 
Moreover, wind direction and intensity can change unpredictably, and especially for long endurance mission planning, there is considerable uncertainty about wind conditions to be encountered.
The prior approaches are not robust against changes in wind conditions, exposing a mission liability. 
Therefore, there is a need to be able to evaluate terminal safe conditions robust to wind uncertainty. 

% \cite{schopferer2018path} consider wind uncertainty by considering the worst-case curvature constraint in a tailwind condition. However, this results in overly conservative curvature which limit deployment to large open areas. 

\begin{figure}[t]
\centerline{\includegraphics[width=\linewidth]{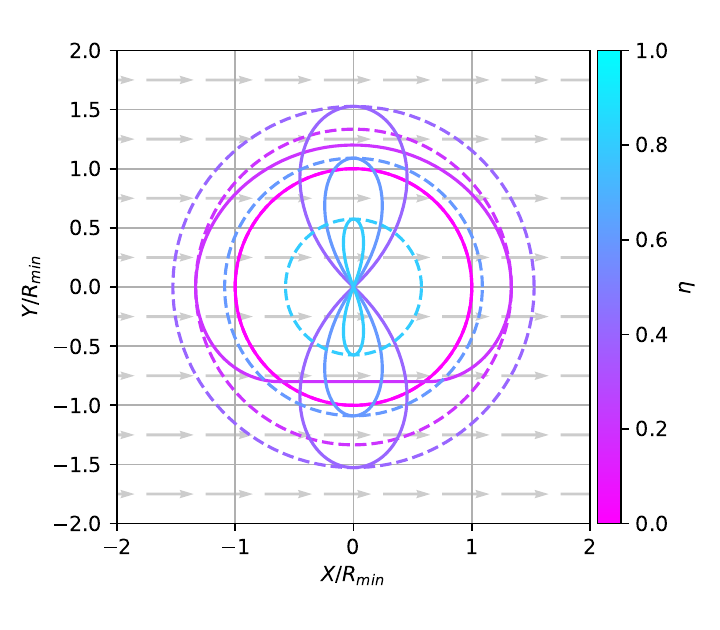}}
\caption{Overview of the safe valid set and example of periodic paths given the wind direction. Our approach defines a wind invariant safe periodic set which allow the vehicle to stay within the region given any wind speed and wind direction.}
\label{fig:combined_maximum_extent}
\end{figure}

In this work, we extend the safety evaluation using periodic paths in~\cite{lim2024safe} to consider wind.
We define the minimum-extent wind-invariant safe set as the region where a periodic trochoidal path always exists, regardless of wind conditions.
This is done by identifying key patterns of periodic trochoidal paths and finding the minimum radius that contains at least one periodic path for every wind condition.
We show that by considering multiple periodic path types simultaneously, we can reduce the required radius by \qty{20}{\%}, enlarging the operational space of fixed-wing vehicles to safely reach more diverse regions in steep mountainous terrain.

Our work enables the identification of regions above the terrain where a safe trochoidal path for all wind conditions may be determined. The key contributions of this paper are:
\begin{itemize}
    \item Identification of periodic trochoidal paths for generating kinematically feasible periodic paths in wind.
    \item Formulation of a minimum-extent wind-invariant safe trochoidal path set.
\end{itemize}

\section{Related Work}
\subsection{Safe Planning}
Evaluating the safety of systems with dynamic constraints requires evaluating \ac{ICS}~\cite{fraichard2004inevitable}.
\ac{ICS} denote regions of the state space where, once entered, it is not possible to avoid a collision with feasible inputs.
Evaluating \ac{ICS} can be impractical, as infinite horizon collision checks need to be made.
Therefore, practical systems use emergency maneuvers that contain a final stop condition~\cite{tordesillas2019faster}.
However, as fixed-wing aerial vehicles are unable to stop mid-air, it is challenging to design emergency maneuvers in confined spaces.
Another way to approximate \ac{ICS} is to use periodic paths~\cite{bekris2010avoiding}.
Circular periodic paths have been used for evaluating safety~\cite{lim2024safe, arora2015emergency}.
In this work, we explore extending circular periodic paths with trochoidal path representations to account for the presence of wind.

\subsection{Wind Aware Navigation}
% Wind-aware planning for aerial vehicles have been considered for uniform wind cases~\cite{mcgee2005optimal, mcgee2007optimal, moon2023time, techy2009minimum, techy2010planar, bucher2023robust} and non-uniform wind~\cite{oettershagen2017towards, schopferer2018path, duan2024energy}. 

Numerous approaches assume that the wind field is uniform to simplify the planning problem.
Trochoidal paths have shown to be an effective way to represent the non-holomonic constraints of a Dubins vehicle~\cite{dubins1957curves} operating in wind~\cite{mcgee2005optimal, mcgee2007optimal,techy2009minimum, moon2023time, techy2010planar}.
While these approaches correctly consider the effect of wind for a curvature-constrained path in the air-relative frame, they assume that the wind field is known and uniform.

Wind can be highly non-uniform when affected by the topography of the terrain~\cite{achermann2024windseer}.
Non-uniform wind fields have been considered for planning by propagating an air-relative Dubins path through the wind field~\cite{oettershagen2017towards} or integrating the energy cost over a ground-relative Dubins airplane path~\cite{duan2024energy}.
However, acquiring accurate non-uniform wind field information can be challenging.
Moreover, it is challenging to evaluate the safety of the planned path due to the high uncertainty of the wind field.
In this work, we consider uniform wind fields, as the scale of periodic paths is small enough that a uniform wind assumption may hold.

Wind uncertainty for planning has been considered by predicting the wind field~\cite{achermann2019learning, luders2016wind}.
However, these approaches require accurate modeling of the wind field.
Wind uncertainties can also be handled robustly by calculating the conservative turn radius based on a worst-case scenario, usually occurring in a tailwind~\cite{schopferer2018path, wolek2015feasible}.
However, this may result in constraining the path to conservative curvature constraints, which may be prohibitive while operating in steep, mountainous environments.
In this work, we propose an efficient way to evaluate terminal safe conditions that are valid for all wind conditions.

\section{Preliminaries}
% In this section, we introduce path representations of Dubins paths and trochoid paths. 
\subsection{Dubins curves}
A planar kinematic model of a fixed-wing vehicle in a wind-free environment can be written as the following, where $V_a$ is the airspeed, $R_{min}$ is the minimum turning radius, and $1/R_{min}$ is the maximum curvature.
\begin{align}
    \Dot{x} &= V_a \cos{\psi}\\
    \Dot{y} &= V_a \sin{\psi}\nonumber\\
    \Dot{\psi}& = u \quad s.t. \quad\|u\| \leq V_a/R_{min}\nonumber
\end{align}

Dubins~\cite{dubins1957curves} showed that the curvature-constrained shortest path between two states consists of a sequence of three segments, consisting of three segments of arcs or straight lines.
% This was further reformulated using Pontryagin's optimality principle~\cite{boissonnat1994shortest} and the optimality conditions shown in~\cite{bui1994shortest}.
A straight segment (S) with an arc segment either a right(R) or left(L) turn creating a Dubins curve, there are six possible path types denoted as $\{RSR, LSL, RSL, LSR, RLR, LRL\}$~\cite{shkel2001classification, lim2023circling}.
% \begin{figure}[t]
% \centerline{\includegraphics[width=\linewidth]{figures/dubins_all_cases.pdf}}
% \caption{(place holder) All four CSC path types for Dubins paths.}
% \label{fig:diagram_figure_eight}
% \end{figure}

\subsection{Trochoids}
A trochoid curve represents the movement of a Dubins vehicle under uniform wind.
% This is useful for representing vehicle movement under environment flows, such as fixed-wing \acp{UAV} in wind.
We define the wind condition through the wind ratio $\eta = V_w/V_a \in [0, 1)$, and wind direction $\psi_w \in [0, 2\pi)$. Without loss of generality, we consider the kinematics in the trochoid frame, where the x-axis is aligned with the wind direction. Therefore, the kinematics can be written as~\refequ{eq:trochoid}, where $\phi_t = \psi-\psi_w$. Note that we only consider wind with $\eta < 1$ in this paper.
\begin{align}
    \Dot{x}_t &= V_a (\cos{\phi} + \eta)\label{eq:trochoid}\\
    \Dot{y}_t &= V_a \sin{\phi}\nonumber\\
    \Dot{\phi}_t& = u \quad s.t. \quad\|u\| \leq V_a/R_{min}\nonumber
\end{align}

A trochoid path can be written as~\refequ{eq:trochoid_path}, where $\omega = V_a/R_{min}$ is the maximum yaw rate, $\delta \in \{-1, 1\}$ the turn direction, and $\bm{x}_0$ the initial configuration of the vehicle at $t=0$.
\begin{align}
    x_{t} =& \frac{V_a}{\delta \omega} \sin(\delta \omega t+ \phi_{t_0}) + V_w t + x_{t_{0}}\label{eq:trochoid_path}\\
    y_{t} =& \frac{-V_a}{\delta \omega} \cos(\delta \omega t+ \phi_{t_0})+ y_{t_{0}}\nonumber\\
    \psi_{t} =& \delta \omega t + \psi_0\nonumber\\
    &x_{t_{0}} = x_0 - V_a/(\delta \omega)\sin(\phi_{t_0})\nonumber\\
    &y_{t_{0}} = y_0 + V_a/(\delta \omega)\cos(\phi_{t_0})\nonumber\\
    &\phi_{t_0} = \psi_0 - \psi_w\nonumber
\end{align}

Similar to Dubins, the time optimal path between the two states $\bm{x}_0 = (x_0, y_0, \psi_0)$, $\bm{x}_f = (x_f, y_f, \psi_f)$ also consists of six path types $\{RSR, LSL, RSL, LSR, RLR, LRL\}$~\cite{techy2009minimum}, where the right and left turn segment consists of trochoids.

\section{Periodic Trochoidal Paths}\label{sec:periodic_trochoid_paths}
A periodic path is obtained when the initial and final configurations of the path are identical, such that the path can be followed repeatedly. 
Therefore, we consider a periodic trochoidal path where $\bm{x}_0 = \bm{x}_f$. As the nonzero time-optimal path between two states would consist of six path types, we divide the characteristics of periodic paths into three cases.

% \begin{figure}[t]
% \centerline{\includegraphics[width=\linewidth]{figures/rsr_periodic_paths.pdf}}
% \caption{LSL/RSR periodic trochoid paths with different wind ratios, tangent to the initial state. Straight segments always point towards the upwind direction.}
% \label{fig:periodic_trochoid_path}
% \end{figure}
% \begin{definition}[Periodic Paths] \label{prop:problem_feasible} 
% Given a set $\mathcal{S}$, and a set that includes all paths $\Gamma$,  
% \end{definition}

\subsection{Periodic RSR/LSL Paths}\label{sec:rsr_lsl_periodic_path}
RSR/LSL-type trochoidal paths have a start and end trochoid segment that turns in the same direction with a straight segment in between. We parameterize the path by the first trochoid path duration $t_A$ starting at $\bm{x}_0$, second trochoid path time $t_B$ that arrives at $\bm{x}_f$ at $t=t_{2\pi}=2\pi/\omega$, and the slope of the line segment $\alpha$.

This type of trochoid path has been shown to have an analytic solution for $t_A, t_B, \alpha$~\cite{techy2009minimum}. The analytic solution can be computed as~\refequ{eq:bsb_analytical_solution}~\cite{techy2009minimum}. Note that there can be multiple solutions depending on the index $k$.
\begin{align}
    \alpha &= \tan^{-1}\left(\frac{y_{t_{20}} - y_{t_{10}}}{x_{t_{20}} - x_{t_{10}} + V_w \frac{\phi_{t_1} - \phi_{t_2} + 2k\pi}{\delta_2 \omega}}\right)\label{eq:bsb_analytical_solution}\\
    t_A &= \frac{t_{2\pi}}{\delta_1 2\pi} \left[ \sin^{-1}\left(\frac{V_w}{V_a}\sin(\alpha)\right) + \alpha - \phi_{t_1} \right]\nonumber\\
    t_B &= t_A + \frac{\phi_{t_1} - \phi_{t_2} + 2k\pi}{\delta_2 \omega} \quad k \in \{-3, -2, -1, 0, 1, 2\}\nonumber
\end{align}

\begin{figure}[t]\centerline{\includegraphics[width=\linewidth]{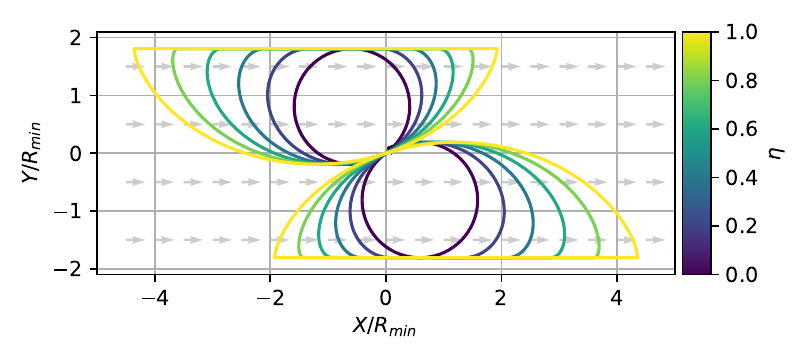}}
\caption{LSL/RSR periodic trochoidal paths with different wind ratios, tangent to the (arbitrary) initial state $\bm{x}_0 = (0.0, 0.0, 0.1\pi$). Wind direction is marked as grey arrows.}
\label{fig:rsr_periodic_path}
\end{figure}

For a periodic path, this means that the start and end arc segments are part of a single arc as ($\delta_1=\delta_2$), and therefore $\alpha = \tan^{-1}\left(0\right)=0$. Applying this simplifies \refequ{eq:bsb_analytical_solution} to \refequ{eq:bsb_periodic_solution}. Therefore, $t_A$ is simply the time it takes the system to turn to face directly upwind or downwind.
\begin{align}
    t_A &= \frac{t_{2\pi}}{\delta_1 2\pi} \left(\alpha-\phi_{t_1}\right)\label{eq:bsb_periodic_solution}\\
    t_B &= t_A + \frac{2k\pi}{\delta_2 \omega} + t_{2\pi} \quad k \in \{-3, -2, -1, 0, 1, 2\}\nonumber
\end{align}
 
\reffig{fig:rsr_periodic_path} shows the evolution of periodic paths with different wind ratios. As the wind speed increases, the circular path extends along the wind direction, and the path forms a mushroom-like shape due to the large drift.

\subsection{Periodic RSL/LSR Paths}\label{sec:rsl_lsr_periodic_path}
LSR/RSL-type trochoidal paths are paths where the start and end trochoid is in the opposite direction ($\delta_1 = -\delta_2$). This form of time optimal path does not have a closed-form solution and is solved numerically~\cite{techy2009minimum}. 

With the same parameterization as in~\refsec{sec:rsr_lsl_periodic_path}, we can find $t_A$ by finding the root of $f(t_A)=0$ in~\refequ{eq:lsr_numerical_solution}. The solution can be found numerically, through the Newton-Raphson method.
\begin{align}
    f(t_A) =& Ecos(\delta_1 \omega t_A + \phi_{t_1}) + F sin(\delta_1 \omega t_A + \phi_{t_1}) - G\\
    E =& V_a \left(V_w \frac{\delta_1 - \delta_2}{\delta_1\delta_2\omega} - (y_{t_{20}} - y_{t_{10}})\right)\nonumber\\
    F =& V_a \left((x_{t_{20}} - x_{t_{10}}) \right.\nonumber\\&\left.+ V_w \left(t_A\left(\frac{\delta_1}{\delta_2} -1\right) + \frac{\phi_{t_1} - \phi_{t_2} + 2k\pi}{\delta_2 \omega} \right) \right)\\
    G =& V_w (y_{t_{20}} - y_{t_{10}}) + \frac{V^2_a (\delta_2 - \delta_1)}{\delta_1\delta_2\omega}
    \label{eq:lsr_numerical_solution}
\end{align}
% Applying the periodic path conditions, we can simplify the above formula to the following.
% \begin{align}
%     f(t_A) &= Ecos(\delta_1 \omega t_A + \phi_{t1}) + F sin(\delta_1 \omega t_A + \phi_{t_1}) - G\\
%     E =& V_a \left(V_w \frac{\delta_1 - \delta_2}{\delta_1\delta_2\omega})\right)\\
%     F =& V_a \left(V_w \left(-2t_A + \frac{\delta_2\omega t_{2\pi}+2k\pi}{\delta_2 \omega} \right) \right)\\
%     G =& \frac{V^2_a(\delta_2 - \delta_1)}{\delta_1\delta_2\omega})
% \end{align}
Once $t_A$ is found, $t_B$ is determined from~\refequ{eq:rsl_rb}.
\begin{align}
    t_B = \frac{\delta_1}{\delta_2}t_A + \frac{k}{\delta_2}t_{2\pi} + t_{2\pi} \quad k \in \{-3, -2, -1, 0, 1, 2\}\label{eq:rsl_rb}
\end{align}
\reffig{fig:lsr_periodic_path} shows the RSL/LSR-type periodic trochoid paths from a given state at the origin with different wind ratios. As the wind ratio increases, the path extends in the wind direction and forms a figure-eight shape path.

\begin{figure}[t]
% \centerline{\includegraphics[width=\linewidth]{figures/lsr_rsl_periodic_paths.pdf}}
\centerline{\includegraphics[width=\linewidth]{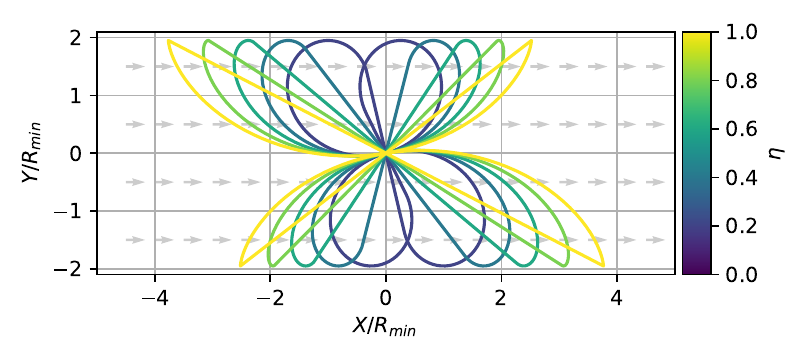}}
% \centerline{\includegraphics[width=\linewidth]{figures/lsr_periodic_paths_0.pdf}}
% \centerline{\includegraphics[width=\linewidth]{figures/lsr_periodic_paths_1.pdf}}
\caption{LSR/RSL periodic trochoidal paths with different wind ratios, tangent to the (arbitrary) initial state $\bm{x}_0 = (0.0, 0.0, 0.1\pi)$. Wind direction is marked as grey arrows.
% Straight segments always point towards the upwind direction.
}
\label{fig:lsr_periodic_path}
\end{figure}
\begin{figure}[t]
\centerline{\includegraphics[width=\linewidth]{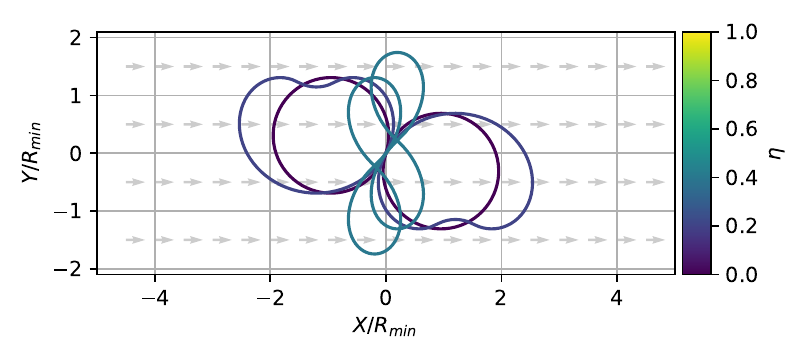}}
\caption{RLR/LRL-type periodic trochoid paths with different wind ratios, tangent to the (arbitrary) initial state $\bm{x}_0 = (0.0, 0.0, 0.4\pi)$. There are no feasible paths for higher wind ratios. The wind direction is marked as grey arrows.}
\label{fig:rlr_periodic_path}
\end{figure}

\subsection{Periodic RLR/LRL Paths}\label{sec:rlr_lrl_periodic_path}
RLR/LRL-type periodic paths consist of three arcs, where the first and last arc segments are in the same direction~($\delta_1=\delta_3$), with the middle arc segment in the opposite direction~($\delta_1=-\delta_3$). Therefore, the periodic path consists of two trochoidal paths in different directions. This path type also needs to be solved numerically~\cite{techy2009minimum}. The numerical solution to this path type is omitted for brevity.

% \begin{align}
%     f(t_A, T) = \begin{pmatrix}
%         \frac{2Va}{\delta_1 \omega} \sin(\delta_1 \omega t_A + \phi_{t_1} + x_{t_{10}} - x_{t_{30}} + \frac{2V_a}{\delta_2 \omega} \sin(\delta_2 \omega \frac{T}{2} + \frac{\psi_f}{2} + \delta_1 \omega t_A + \frac{\phi_{t_1}}{2})\\
%         \frac{2Va}{\delta_1 \omega} \sin(\delta_1 \omega t_A + \phi_{t_1} + x_{t_{10}} - x_{t_{30}} + \frac{2V_a}{\delta_2 \omega} \sin(\delta_2 \omega \frac{T}{2} + \frac{\psi_f}{2} + \delta_1 \omega t_A + \frac{\phi_{t_1}}{2})
%     \end{pmatrix}
% \end{align}

\reffig{fig:rlr_periodic_path} shows the RLR/LRL periodic paths from a given state at the origin with different wind ratios. It can be seen that the shape of the path switches from a mushroom-like shape to a figure-eight path at higher wind ratios. We denote this path shape as figure-eight paths. Also, only paths at lower wind ratios are shown as a feasible solution for higher wind ratios does not exist. 

\section{Minimum-Extent Wind-Invariant Safe Set}
In \refsec{sec:periodic_trochoid_paths}, we have shown the different periodic trochoidal paths that exist from a given initial state $\bm{x}_0$. In this section, we introduce the problem of finding the minimum-extent set, which contains at least one feasible trochoidal periodic path for all wind conditions $\eta \in [0, 1)$. We define this set as the \emph{minimum-extent wind-invariant safe set}. 

\begin{definition}[Radius of Extent] \label{prop:max_safe_set}
Given a periodic path $\Gamma(\eta, \psi_w)$ defined by wind condition $(\eta, \psi_w$), and two arbitrary states $x_i, x_j$ are on the path, we define the radius of the extent as the following.
\begin{align}
    D = \frac{1}{2}\max_{{\forall x_i, x_j \in \Gamma}} \|x_i - x_j\|
\end{align}
\end{definition}

As the radius of the extent is only defined by the shape of the periodic path, a circular region with radius $D$ would be able to contain the periodic path for all wind directions $\psi_w$. Therefore, $D$ is a function of the wind ratio $\eta$.

\begin{definition}[Minimum-Extent Wind-Invariant Safe Set] \label{prop:problem_feasible} 
Given that a set $\mathcal{S}\subset\mathbb{R}^2$ defined as the area containing at least one periodic trochoidal path for all wind conditions $\eta \in [0, 1)$, $\psi_w \in [0, 2\pi)$ exists, we define the \emph{minimum-extent wind-invariant safe set} as the set with the smallest maximum radius of extent.
\begin{align}
    D_{min}=\max_{\eta \in [0, 1)} D(\eta)\nonumber
\end{align}
\end{definition}

In summary, at least one trochoidal periodic path exists in the minimum-extent wind-invariant safe set for all wind ratio $\eta=[0, 1)$ and wind direction $\psi=[0, 2\pi)$.

\subsection{RSR/LSL Wind Invariant Set}
\begin{figure}[t]
\centerline{\includegraphics[width=\linewidth]{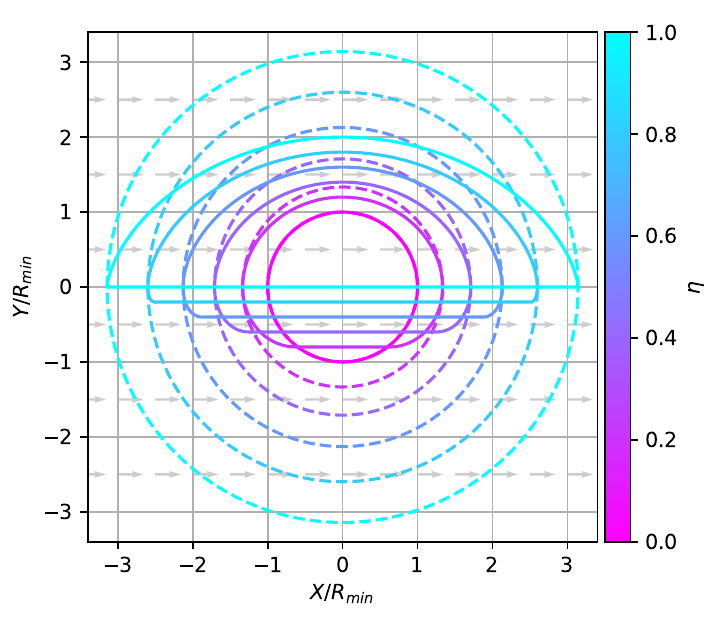}}
\caption{Evolution of RSR/LSL periodic trochoidal paths with different wind ratios. The radius of maximum extent for the minimum-extent wind-invariant set is $\pi R$.}
\label{fig:rsr_wind_invariant_set}
\end{figure}

Given that RSR/LSL-type trochoidal periodic paths consist of one continuous arc segment and a line segment connecting the arc, the radius of extent can be calculated by computing the distance between the extrema of the trochoid, that are, the minimum and maximum x-coordinates. We can consider a trochoid starting and ending at $\phi_{t_0} = \pi$. 
% Applied to trochoidEq.(3), delta = -1
% \begin{align}
%     x_{t} =& -\frac{V_a}{\omega} \sin(-\omega t+ \pi) + V_w t + x_{t_{0}}\\
%     y_{t} =& \frac{V_a}{\omega} \cos(- \omega t+ \pi)+ y_{t_{0}}\nonumber\\
%     &x_{t_{0}} = x_0\nonumber\\
%     &y_{t_{0}} = y_0 + V_a/\omega\nonumber\\
%     \phi_{t} =& \pi\nonumber
% \end{align}
The extrema are the points where the x-direction velocity diminishes in~\refequ{eq:trochoid}.
\begin{align}
    \Dot{x} = V_a(\cos\phi + \eta) = 0
\end{align}
% \begin{align}
%     \cos\Bar{\phi} + \eta = 0\\
%     \Bar{\phi} = \arccos(-\eta), 2\pi - \arccos(-\eta)\\
%      t = (\arccos(-\eta)-\pi)/\omega
% \end{align}

% \begin{align}
%     x_{t1} =& -\frac{V_a}{\omega} \sqrt{1-\eta^2} + V_w ((\arccos(-\eta)-\pi)/\omega) + x_{t_{0}}\\
%     x_{t2} =& \frac{V_a}{\omega} \sqrt{1-\eta^2} + V_w ((\pi-\arccos(-\eta))/\omega) + x_{t_{0}}\\
%     \frac{\Delta x}{2} =& \frac{V_a}{\omega} \sqrt{1-\eta^2} + \frac{V_w}{\omega}(\pi - \arccos{(-\eta)})\\
%     \frac{\Delta x}{2} =& \frac{V_a}{\omega} (\sqrt{1-\eta^2} + \eta(\pi - \arccos{(-\eta)}))
% \end{align}

Calculating the distance between the extrema using~\refequ{eq:trochoid_path}, the radius for the minimum-extent wind-invariant safe set can be found~\refequ{eq:max_extent_rsr}.
\begin{align}
    D(\eta) = \frac{V_a}{\omega} \left(\sqrt{1-\eta^2} + \eta(\pi - \arccos{(-\eta)})\right)
    \label{eq:max_extent_rsr}
\end{align}

To find the respective minimum-extent wind-invariant safe set, we place the periodic path to be centered around the origin by placing the trochoid as $x_{t0} = 0$, $y_{t0} = -\eta V_a/\omega$ in~\refequ{eq:trochoid_path}. \reffig{fig:rsr_wind_invariant_set} shows the wind-invariant safe set with the center of the safe set at the origin. From~\refequ{eq:max_extent_rsr}, $V_w \approx 0 \rightarrow D \approx V_a/\omega$, and $V_w \approx V_a \rightarrow D \approx \pi V_a/\omega$. Note that the path stays in its ``mushroom" like form, and the required radius monotonically increases as the wind ratio increases. Therefore, the radius of the minimum-extent wind-invariant set is $D_{min} = \pi R_{min}$. 

\subsection{RSL/LSR Wind Invariant Sets}
We have shown in \refsec{sec:rsl_lsr_periodic_path} that it is possible to find an RSL/LSR-type trochoidal periodic path. In this section, we show that RSL/LSR-type periodic trochoidal paths with minimum extent have a straight segment that is zero length, which becomes equivalent to the minimum extent of RLR/LRL-type trochoid paths. We define this special case of trochoidal periodic path as a \emph{figure-eight path}.

\begin{figure}[b]
\centerline{\includegraphics[width=0.5\linewidth]{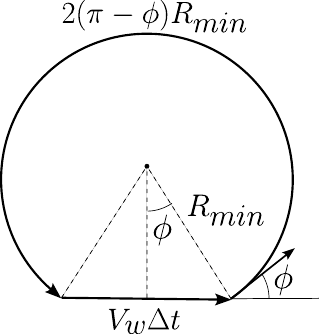}}
\caption{Diagram for the condition for the figure eight path.}
\label{fig:diagram_figure_eight}
\end{figure}

\begin{proposition}[Existence of Figure Eight] \label{prop:problem_feasible} Given a wind condition $\eta$, $\psi_w$, there exists an initial heading $\Bar{\phi} = \Bar{\psi}_0 - \psi_w$ for RSL/LSR-type trochoidal periodic paths, in which the length of the straight segment is zero.
\end{proposition}
\begin{proof}
Assume that some heading $\Bar{\phi}$ creates a figure-eight trochoidal periodic path. In the air-relative fame, a figure-eight path is formed when the time to traverse the arc segment $t_A$ is identical to the drift caused by the wind~(\reffig{fig:diagram_figure_eight}).
\begin{align}
    t_A = \frac{R_{min}(2\pi-2\Bar{\phi})}{V_a} = \frac{2R_{min}sin\Bar{\phi}}{V_w}\quad R_{min} = \frac{V_a}{\omega}\label{eq:condition_figure_eight}
\end{align}

By rearranging \refequ{eq:condition_figure_eight}, we get~\refequ{eq:root_finding_figure_eight}, where $\Bar{\phi}$ is the solution if exists.
\begin{align}
    f(\Bar{\phi}) = \eta (\pi - \Bar{\phi}) - \sin{\Bar{\phi}} = 0\label{eq:root_finding_figure_eight}
\end{align}

To determine whether a solution exists, $f(0) = \eta \pi > 0$, and $f(\pi) = 0$, and $\left.\frac{\partial f(\psi)}{\partial \psi}\right\rvert_{\psi=\pi} > 0$.
% \begin{align}
%     \left.\frac{\partial f(\psi)}{\partial \psi}\right\rvert_{\psi=\pi} =  - \eta + \cos(\pi) > 0
% \end{align}
Therefore, with the intermediate value theorem, we can conclude that there should be at least one $\Bar{\phi}$ that creates a figure-eight path.
\end{proof}

Through empirical evaluation, we observed the figure-eight path type to be the path of minimum extent for RSL/LSR-type periodic paths. Therefore, we only need to consider the figure-eight path to find the radius of the minimum-extent wind-invariant set. 
% \begin{proposition}[Minimum Extent Figure Eight] \label{prop:rsl_minimum_extent} 
% Given a fixed wind condition $\eta$, the periodic path of RSL/LSR path types with minimum extent have zero-length straight segments.
% \end{proposition}
% \begin{proof}
% To prove this, we divide the case where $\psi > \Bar{\psi}$ and $\psi < \Bar{\psi}$. 

% In case where $\psi < \Bar{\psi}$, the extent is always larger. ($\psi \in [0. \frac{1}{2}\pi)$)
% \begin{align}
%     D = R + R\cos\psi \geq R + R\cos{\Bar{\psi}}\nonumber
% \end{align}

% In case where $\psi > \Bar{\psi}$, there does not exist a solution for figure eight, and therefore it is required that an additional loop is formed. Therefore, the minimum extent is achieved when $\psi = \Bar{\psi}$.
% \begin{align}
%    (S\cos(\delta \omega t_A+ \phi_{t}))/V_w = (S + V_a 2t_A)/V_a\\
%    S = \frac{2 \eta V_a t_A - 2R(\sin(\phi_t) - 2R\cos(\omega t_A + \phi_t))}{\cos(\delta \omega t_A+ \phi_{t}) - \eta}
% \end{align}
% \end{proof}
% \begin{itemize}
%     \color{red}
%     \item TODO: Prove that RSL/LSR paths would always have a bigger extent than figure eight paths.
% \end{itemize}
% Given that in \refsec{sec:rlr_lrl_periodic_path}, we have shown that RLR/LRL-type periodic trochoid paths always consist of a figure eight path, the proposition implies that RSL/LSR path types do not need to be considered to find the minimum extent of periodic trochoid paths.
For the figure-eight paths, we can find the radius of the figure-eight path by revisiting \reffig{fig:diagram_figure_eight}. Note that $\Bar{\phi}$ does not have a closed-form solution and therefore needs to be found numerically.

\begin{align}
    D =& \frac{V_a}{\omega}\left(1 + \cos\Bar{\phi} \right)\label{eq:extent_figure_eight}\\
    & s.t. \quad \eta (\pi - \Bar{\phi}) - \sin{\Bar{\phi}} = 0\nonumber
\end{align}

\begin{figure}[t]
\centerline{\includegraphics[width=\linewidth]{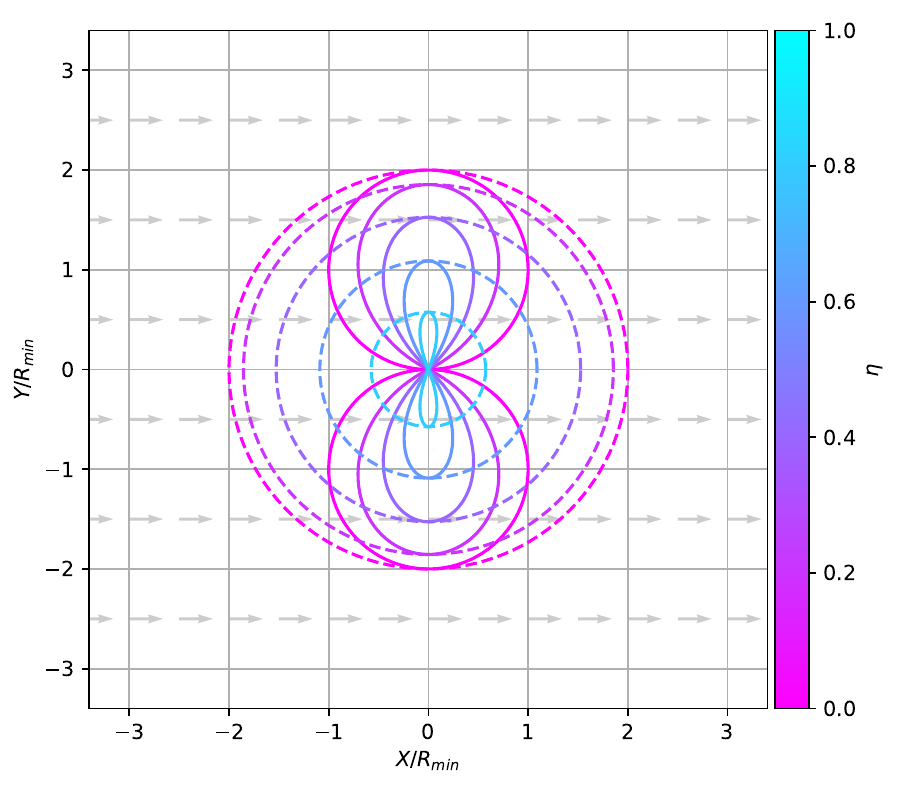}}
\caption{Wind invariant set in RLR/LRL periodic trochoid paths. The maximum extent decreases as wind ratio increases.}
\label{fig:lsr_wind_invariant_set}
\end{figure}

\reffig{fig:lsr_wind_invariant_set} shows the evolution of the figure-eight paths for different winds. As $V_w \approx 0$, $\Bar{\phi} \approx 0$ and $D \approx 2R_{min}$, and as $V_w \approx 0$, $\Bar{\phi} \approx \pi$. At $\eta \approx 0$, the figure eight becomes two tangential circles, and at $\eta\approx1$, the figure-eight path becomes a point where the vehicle can be stationary as the wind speed matches the vehicle's airspeed. Note that in contrast to RSR/LSL periodic trochoid paths, the maximum extent of LSR/RSL-type periodic trochoid paths decreases with the wind ratio increasing. Therefore, the radius of the minimum-extent wind-invariant RSL/LSR-type paths are $D_{min} = 2R_{min}$.

\subsection{RLR/LRL Wind Invariant Sets}
From the \refsec{sec:rlr_lrl_periodic_path}, we have shown that the path with minimum extent of the RLR/LRL-type periodic trochoidal path is in the form of a figure eight or a curved mushroom shape. For the curved mushroom-shape RLR/LRL-type paths, the radius of the wind invariant safe set is identical to the RSR/LSL-type trochoidal periodic path, as the long arc segment defines the extent. For the figure-eight RLR/LRL-type paths, the radius of the wind invariant safe set would be identical to the RSL/LSR-type periodic paths. Therefore, the RLR/LRL-type paths are redundant to the calculation of the minimum-extent wind-invariant set.

\begin{figure*}[t]
\begin{subfigure}{0.33\textwidth}
    \includegraphics[width=\linewidth]{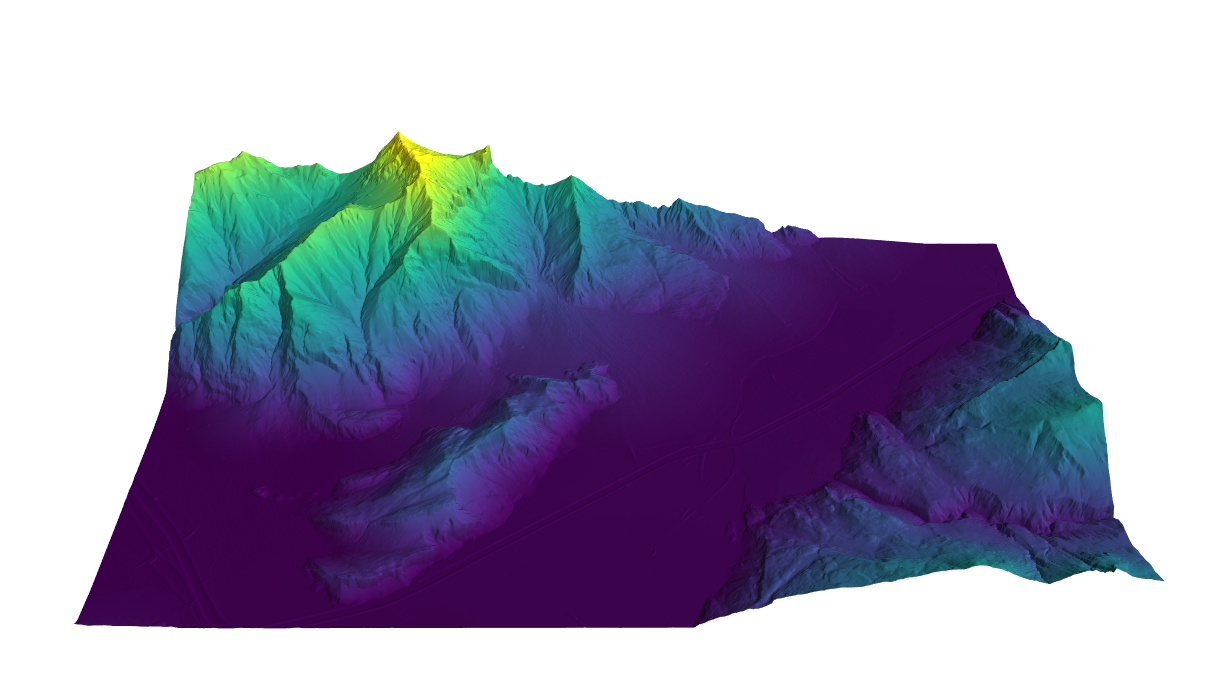}
    \includegraphics[width=\linewidth]{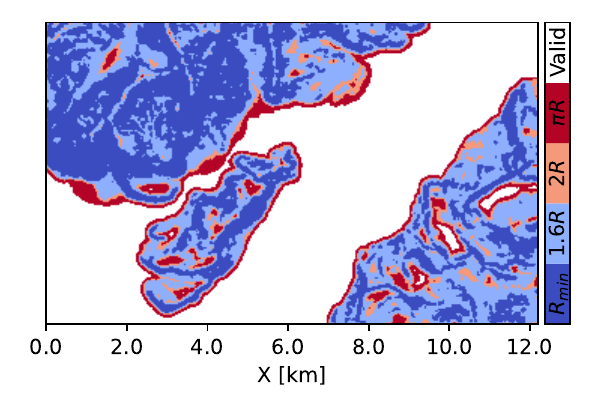}
    \caption{\emph{Sargans}}
    \label{fig:evaluation_terrain:a}
\end{subfigure}
\begin{subfigure}{0.33\textwidth}
    \includegraphics[width=\linewidth]{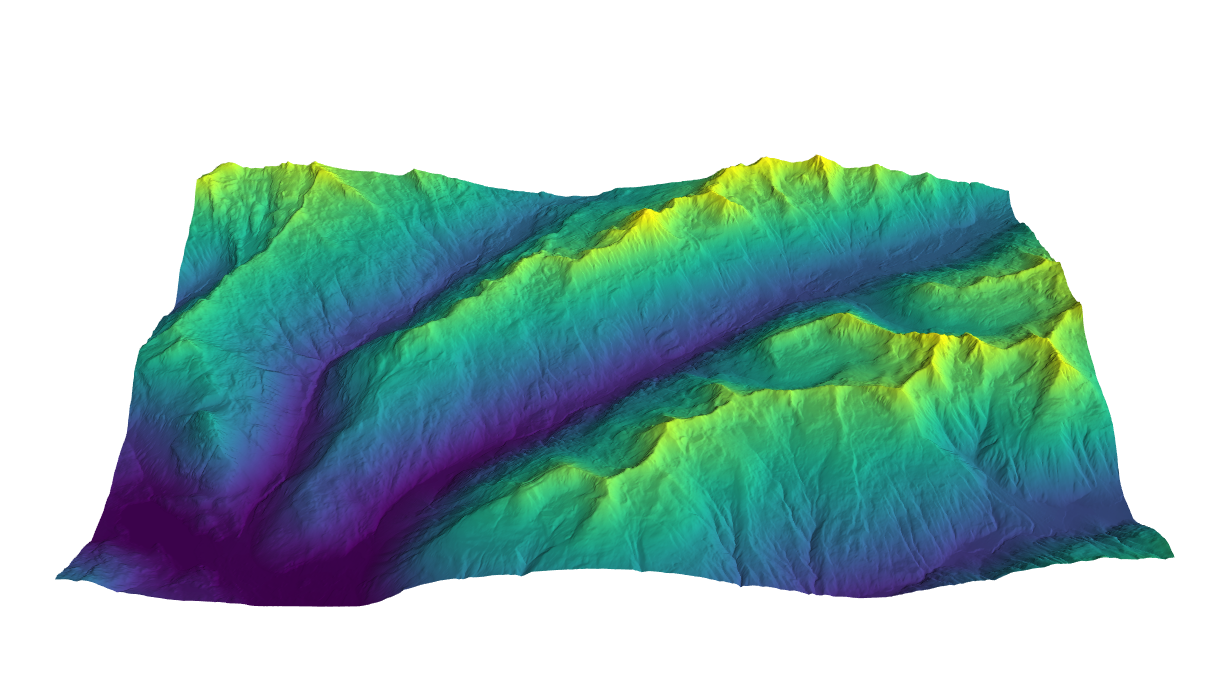}
    \includegraphics[width=\linewidth]{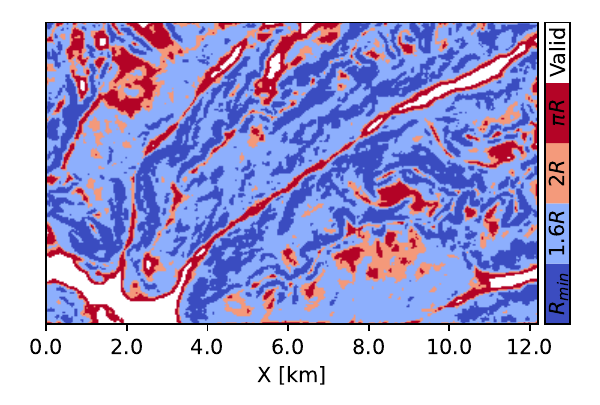}
    \caption{\emph{Dischma Valley}}\label{fig:evaluation_terrain:b} \par 
\end{subfigure}%
\begin{subfigure}{0.33\textwidth}
    \includegraphics[width=\linewidth]{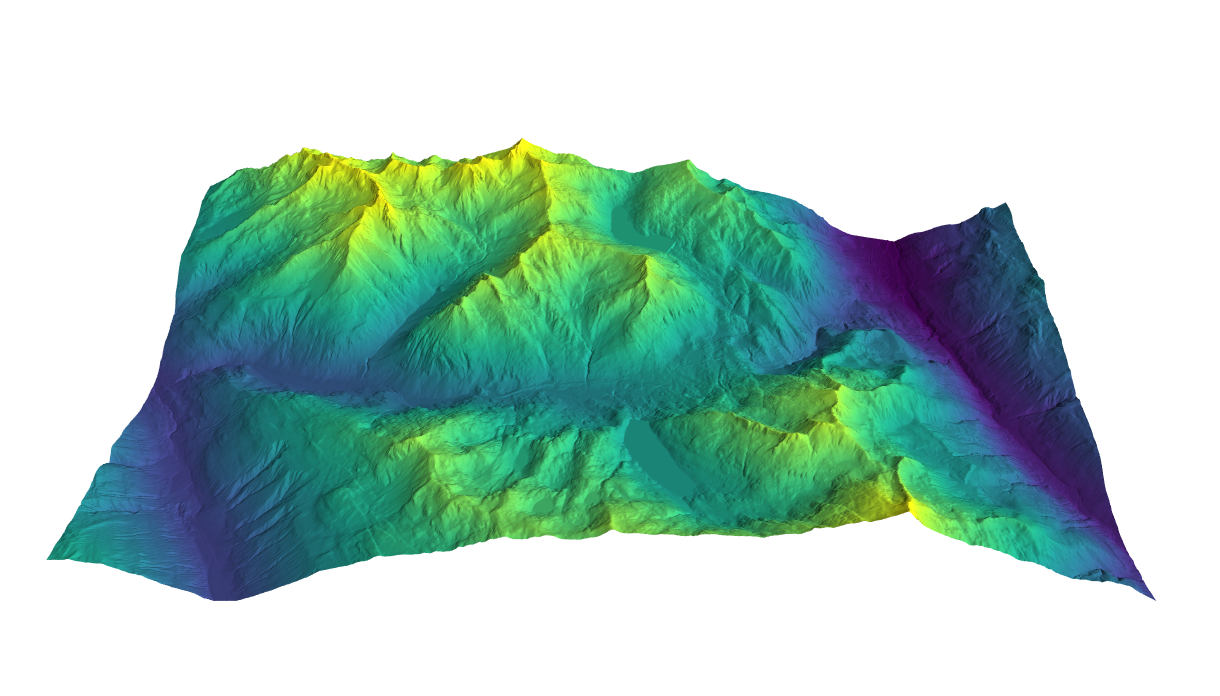}
    \includegraphics[width=\linewidth]{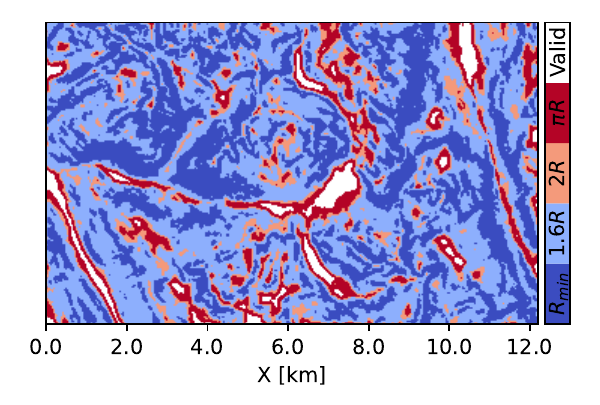}
    \caption{\emph{Gotthard Pass}}\label{fig:evaluation_terrain:c} \par
\end{subfigure}%
\caption{(above) 3D visualization of the elevation data colored with elevation. (below) Two-dimensional projection of the valid regions calculated from the radius of the minimum-extent wind-invariant set in three environments: (a) \emph{Sargans}, (b) \emph{Dischma Valley}, and (c) \emph{Gotthard Pass}. The invalid regions are marked as blue ($R_{min}$), light blue ($1.6R_{min}$, \emph{Ours}), light red ($2R_{min}$), and red ($\pi R_{min}$). The regions that are always valid are marked as white.}
\label{fig:evaluation_terrain}
\end{figure*}

\begin{figure}[t]
\centerline{\includegraphics[width=\linewidth]{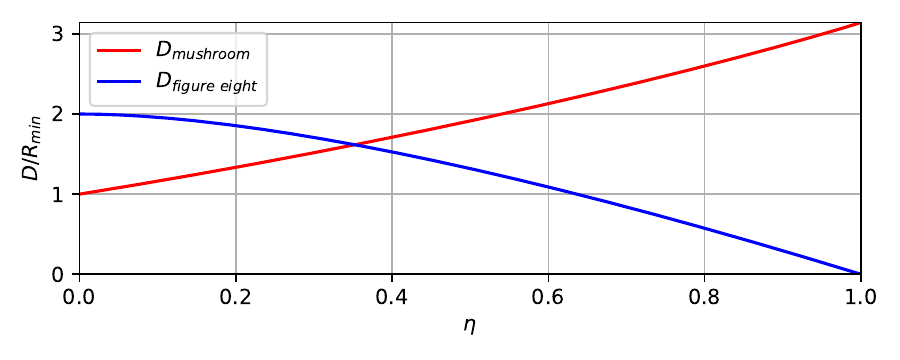}}
\caption{Evolution of minimum radius with increasing wind ratio for mushroom and figure-eight type paths. Extent of $\pi R_{min}$ is required for mushroom shape periodic trochoid paths, and extent of $2R_{min}$ for figure eight type trochoid paths to stay within the boundary.}
\label{fig:windratio_maximum_extent}
\end{figure}

\subsection{Minimum Wind Invariant Set}
We have found that for the RSR/LSL-type periodic trochoid paths, the radius for the minimum wind invariant set is $\pi R_{min}$, and for the RSL/LSR-type periodic trochoid paths it is $2R_{min}$. A key insight is that, with $\eta$ increasing from zero to one, the extent of RSR/LSL-type paths monotonically increases from $R_{min}$ to $\pi R_{min}$, and the maximum extent of the RSL/LSR-type paths monotonically decreases from $2R_{min}$ to $0$~(\reffig{fig:windratio_maximum_extent}). 
By considering a switching logic to select the path type with the smallest extent as function of the wind ratio, we can therefore minimize the extent of the wind invariant set.

% \begin{figure}[h]
% \centerline{\includegraphics[width=\linewidth]{figures/overview.pdf}}
% \caption{Combined periodic path set with switching strategy.}
% \label{fig:combined_maximum_extent}
% \end{figure}
The switching point is found numerically at around $\eta\approx0.35$. For $\eta<0.35$, RSR/LSL mushroom-type paths are selected, and above, RLR/LRL figure-eight type paths are to be used. At the switching point, the extent becomes maximal, attaining a value of $1.62R_{min}$.
%Therefore, we can numerically find the intersection of the two curves around $1.62R_{min}$, where either a mushroom trochoid path or a figure-eight trochoidal path exists for all wind conditions. 
Employing this switching yields a significant reduction of the wind invariant sets' extent (i.e. to $1.62R_{min}$) compared to using a single path type only. Namely, a reduction by \qty{48}{\percent} compared to the mushroom-type trochoid path and by \qty{19}{\percent} compared to figure-eight paths. The periodic path with minimum extent with the switching strategy is visualized in~\reffig{fig:combined_maximum_extent}. 
%As the wind ratio increases, the path type switches around $\eta\approx0.35$.

\section{Evaluation}
\subsection{Setup}
We evaluate the impact of our approach by evaluating the valid flyable regions with the tightened radius of the minimum-extent wind-invariant set. We use three environments, denoted as \emph{Sargans}, \emph{Dischma Valley}, \emph{Gotthard Pass} as in~\cite{lim2024safe}. \reffig{fig:evaluation_terrain} shows the selected \acp{DEM} from SwissAlti3D~\cite{swisstopo2023swissalti3d} with the extent of $12.2\si{km} \times 7.48\si{km}$ with \SI{10}{\metre} lateral resolution. 
The \ac{DEM} uses CH1903/LV03 for lateral coordinates and Bessel 1841 for the vertical datum. These environments are identical in extent but have varying rugged topographies, with the \emph{Gotthard Pass} environment being the most challenging by having the least available valid loiter regions due to the rugged terrain.

For non-flat terrain, the flyable regions are determined by whether the minimum-extent wind-invariant set exists, such that it fits within the distance constraints to the sloped terrain. The maximum distance constraints are defined as \SI{120}{\metre} according to the EU regulations~\cite{eu2019commission}, and the minimum as \SI{50}{\metre} to account for vegetation and unknown obstacles. The minimum (air-relative) turn radius of the vehicle is assumed to be $R_{min}=$\SI{66.67}{\metre}, representing the tilt-rotor VTOL platform in~\cite{lim2024safe}. %Consequently, the steeper the terrain, the more constraining the minimum and maximum terrain distance become.

We compare our approach with three baselines. First, we compare our approach to a conservative Dubins strategy~\cite{schopferer2018path, wolek2015feasible}, where a conservative turn radius accounts for the wind. Since we consider wind conditions of $\eta=[0, 1)$, a conservative turn curvature would be to decrease the curvature by half ($2R_{min})$). This is equivalent to the minimum-extent wind-invariant set of figure-eight path. Second, we compare our method to the radius required by the mushroom-type trochoidal paths with a radius of $\pi R_{min}$. Lastly, we compare how much more radius is required compared to the naive, zero-wind \emph{valid loiter regions} with minimum turn radius~\cite{lim2024safe}.

\subsection{Results}
\reffig{fig:evaluation_terrain} shows the different invalid regions for different radii of different approaches. The invalid regions with a smaller radius are a subset of the invalid regions with a larger radius. Having a smaller region that is marked as invalid means that more locations are available for the fixed-wing aerial vehicle for safe station-keeping in uncertain wind.

As expected, steeper regions are classified as invalid regions. In the \emph{Dischma Valley} environment, it can be seen that a lot of isolated regions become invalid when requiring a radius larger than the proposed $1.6R_{min}$. This means that there are less reachable spaces that the vehicle can reach as a target position. Moreover, valid regions can be used as rally points, which are locations that the vehicle can reach and hold in case of aborting a mission. Note that in the \emph{Gotthard Pass} environment, there is very little valid space left without using the proposed strategy (i.e. single-path invariant sets, light red, red). This means that the vehicle would have very limited options for station keeping, resulting in huge detours in case a mission needs to be paused or aborted. 

\begin{table}[t]
    \centering
    \caption{Coverage of valid regions in the three environments with different radius.}
    \begin{tabular}{c |c| c c c}
        \toprule
        Environment & $R_{min}$\cite{lim2024safe} & $2R_{min}$ \cite{schopferer2018path} & $\pi R_{min}$ &  \emph{Ours}\\
        % \midrule
        % Extent radius & $R$ & $\pi R$ & $2 R$ & $1.6 R$\\
        \midrule
        \emph{Sargans} & 0.7701 & 0.4570 & 0.3942 & \textbf{0.5102}\\
        \emph{Dischma Valley} & 0.7910 & 0.1385 & 0.0353 & \textbf{0.2718}\\
        \emph{Gotthard Pass} & 0.7317 & 0.1187 & 0.0240 & \textbf{0.2259}\\
        \bottomrule
    \end{tabular}
    \label{tab:valid_regions}
\end{table}

We quantify the coverage of valid regions to evaluate how much of the space is reachable (\reftab{tab:valid_regions}). It can be seen that the proposed approach classifies valid regions of over \qty{20}{\percent} for all environments. However, the valid regions for the feasible Dubins ($2R_{min}$) and mushroom-type paths ($\pi R_{min}$) are significantly lower.

\section{Conclusions}
In this work, we defined the minimum-extent wind-invariant safe set, where at least one periodic trochoidal path exists for all wind conditions.
We show that a tighter radius ($1.62R_{min}$) is achievable using a switching strategy, as compared to naively reducing the curvature or considering a single trochoidal path type only.
It was shown that the tighter radius results in significantly more reachable space when evaluated with real terrain data from mountainous environments. 

We expect our approach to enable the operations of autonomous fixed-wing systems in more challenging wind conditions.
The safe regions that are identified with the wind-invariant set can be used to determining safe terminal paths when planning in wind.
However, planning a path that can reach the safe periodic path while being robust against wind uncertainty is still an open question.
Further, as the periodic trochoidal paths are only feasible for a specific wind condition, the deviation of the vehicle following the path under varying wind conditions would need to be accounted to ensure that the vehicle stays within the distance constraints.

% \section*{ACKNOWLEDGMENT}

% The preferred spelling of the word ÒacknowledgmentÓ in America is without an ÒeÓ after the ÒgÓ. Avoid the stilted expression, ÒOne of us (R. B. G.) thanks . . .Ó  Instead, try ÒR. B. G. thanksÓ. Put sponsor acknowledgments in the unnumbered footnote on the first page.

%%%%%%%%%%%%%%%%%%%%%%%%%%%%%%%%%%%%%%%%%%%%%%%%%%%%%%%%%%%%%%%%%%%%%%%%%%%%%%%%

% References are important to the reader; therefore, each citation must be complete and correct. If at all possible, references should be commonly available publications.

\bibliographystyle{IEEEtran}
\balance
\bibliography{references}

\begin{thebibliography}{10}
\providecommand{\url}[1]{#1}
\csname url@samestyle\endcsname
\providecommand{\newblock}{\relax}
\providecommand{\bibinfo}[2]{#2}
\providecommand{\BIBentrySTDinterwordspacing}{\spaceskip=0pt\relax}
\providecommand{\BIBentryALTinterwordstretchfactor}{4}
\providecommand{\BIBentryALTinterwordspacing}{\spaceskip=\fontdimen2\font plus
\BIBentryALTinterwordstretchfactor\fontdimen3\font minus \fontdimen4\font\relax}
\providecommand{\BIBforeignlanguage}[2]{{%
\expandafter\ifx\csname l@#1\endcsname\relax
\typeout{** WARNING: IEEEtran.bst: No hyphenation pattern has been}%
\typeout{** loaded for the language `#1'. Using the pattern for}%
\typeout{** the default language instead.}%
\else
\language=\csname l@#1\endcsname
\fi
#2}}
\providecommand{\BIBdecl}{\relax}
\BIBdecl

\bibitem{bircher2016threedimensional}
A.~Bircher, M.~Kamel, K.~Alexis, M.~Burri, P.~Oettershagen, S.~Omari, T.~Mantel, and R.~Siegwart, ``Three-dimensional coverage path planning via viewpoint resampling and tour optimization for aerial robots,'' \emph{Autonomous Robots}, vol.~40, pp. 1059--1078, 2016.

\bibitem{oettershagen2018robotic}
P.~Oettershagen, T.~Stastny, T.~Hinzmann, K.~Rudin, T.~Mantel, A.~Melzer, B.~Wawrzacz, G.~Hitz, and R.~Siegwart, ``Robotic technologies for solar-powered uavs: Fully autonomous updraft-aware aerial sensing for multiday search-and-rescue missions,'' \emph{Journal of Field Robotics}, vol.~35, no.~4, pp. 612--640, 2018.

\bibitem{jouvet2019high}
G.~Jouvet, Y.~Weidmann, E.~Van~Dongen, M.~P. L{\"u}thi, A.~Vieli, and J.~C. Ryan, ``High-endurance uav for monitoring calving glaciers: Application to the inglefield bredning and eqip sermia, greenland,'' \emph{Frontiers in Earth Science}, vol.~7, p. 206, 2019.

\bibitem{buhler2017photogrammetric}
Y.~B{\"u}hler, M.~S. Adams, A.~Stoffel, and R.~Boesch, ``Photogrammetric reconstruction of homogenous snow surfaces in alpine terrain applying near-infrared uas imagery,'' \emph{International Journal of Remote Sensing}, vol.~38, no. 8-10, pp. 3135--3158, 2017.

\bibitem{bry2015aggressive}
A.~Bry, C.~Richter, A.~Bachrach, and N.~Roy, ``Aggressive flight of fixed-wing and quadrotor aircraft in dense indoor environments,'' \emph{The International Journal of Robotics Research}, vol.~34, no.~7, pp. 969--1002, 2015.

\bibitem{oettershagen2017towards}
P.~Oettershagen, F.~Achermann, B.~M{\"u}ller, D.~Schneider, and R.~Siegwart, ``Towards fully environment-aware uavs: Real-time path planning with online 3d wind field prediction in complex terrain,'' \emph{arXiv preprint arXiv:1712.03608}, 2017.

\bibitem{lim2024safe}
J.~Lim, F.~Achermann, R.~Girod, N.~Lawrance, and R.~Siegwart, ``Safe low-altitude navigation in steep terrain with fixed-wing aerial vehicles,'' \emph{IEEE Robotics and Automation Letters}, vol.~9, no.~5, pp. 4599--4606, 2024.

\bibitem{mcgee2005optimal}
T.~McGee, S.~Spry, and K.~Hedrick, ``Optimal path planning in a constant wind with a bounded turning rate,'' in \emph{AIAA guidance, navigation, and control conference and exhibit}, 2005, p. 6186.

\bibitem{mcgee2007optimal}
T.~G. McGee and J.~K. Hedrick, ``Optimal path planning with a kinematic airplane model,'' \emph{Journal of guidance, control, and dynamics}, vol.~30, no.~2, pp. 629--633, 2007.

\bibitem{moon2023time}
B.~Moon, S.~Sachdev, J.~Yuan, and S.~Scherer, ``Time-optimal path planning in a constant wind for uncrewed aerial vehicles using dubins set classification,'' \emph{IEEE Robotics and Automation Letters}, 2023.

\bibitem{techy2009minimum}
L.~Techy and C.~A. Woolsey, ``Minimum-time path planning for unmanned aerial vehicles in steady uniform winds,'' \emph{Journal of guidance, control, and dynamics}, vol.~32, no.~6, pp. 1736--1746, 2009.

\bibitem{techy2010planar}
L.~Techy, C.~A. Woolsey, and K.~A. Morgansen, ``Planar path planning for flight vehicles in wind with turn rate and acceleration bounds,'' in \emph{2010 IEEE international conference on robotics and automation}.\hskip 1em plus 0.5em minus 0.4em\relax IEEE, 2010, pp. 3240--3245.

\bibitem{bucher2023robust}
T.~Bucher, T.~Stastny, S.~Verling, and R.~Siegwart, ``Robust wind-aware path optimization onboard small fixed-wing uavs,'' in \emph{AIAA SCITECH 2023 Forum}, 2023, p. 2640.

\bibitem{schopferer2018path}
S.~Schopferer, J.~S. Lorenz, A.~Keipour, and S.~Scherer, ``Path planning for unmanned fixed-wing aircraft in uncertain wind conditions using trochoids,'' in \emph{2018 International Conference on Unmanned Aircraft Systems (ICUAS)}.\hskip 1em plus 0.5em minus 0.4em\relax IEEE, 2018, pp. 503--512.

\bibitem{duan2024energy}
Y.~Duan, F.~Achermann, J.~Lim, and R.~Siegwart, ``Energy-optimized planning in non-uniform wind fields with fixed-wing aerial vehicles,'' in \emph{2024 IEEE/RSJ International Conference on Intelligent Robots and Systems (IROS)}, 2024, pp. 3116--3122.

\bibitem{fraichard2004inevitable}
T.~Fraichard and H.~Asama, ``Inevitable collision states—a step towards safer robots?'' \emph{Advanced Robotics}, vol.~18, no.~10, pp. 1001--1024, 2004.

\bibitem{tordesillas2019faster}
J.~Tordesillas, B.~T. Lopez, and J.~P. How, ``Faster: Fast and safe trajectory planner for flights in unknown environments,'' in \emph{2019 IEEE/RSJ international conference on intelligent robots and systems (IROS)}.\hskip 1em plus 0.5em minus 0.4em\relax IEEE, 2019, pp. 1934--1940.

\bibitem{bekris2010avoiding}
K.~E. Bekris, ``Avoiding inevitable collision states: Safety and computational efficiency in replanning with sampling-based algorithms,'' in \emph{Workshop on Guaranteeing Safe Navigation in Dynamic Environments. In: International Conference on Robotics and Automation (ICRA-10)}, 2010.

\bibitem{arora2015emergency}
S.~Arora, S.~Choudhury, D.~Althoff, and S.~Scherer, ``Emergency maneuver library-ensuring safe navigation in partially known environments,'' in \emph{2015 IEEE international conference on robotics and automation (ICRA)}.\hskip 1em plus 0.5em minus 0.4em\relax IEEE, 2015, pp. 6431--6438.

\bibitem{dubins1957curves}
L.~E. Dubins, ``On curves of minimal length with a constraint on average curvature, and with prescribed initial and terminal positions and tangents,'' \emph{American Journal of mathematics}, vol.~79, no.~3, pp. 497--516, 1957.

\bibitem{achermann2024windseer}
F.~Achermann, T.~Stastny, B.~Danciu, A.~Kolobov, J.~J. Chung, R.~Siegwart, and N.~Lawrance, ``Windseer: real-time volumetric wind prediction over complex terrain aboard a small uncrewed aerial vehicle,'' \emph{Nature Communications}, vol.~15, no.~1, p. 3507, 2024.

\bibitem{achermann2019learning}
F.~Achermann, N.~R. Lawrance, R.~Ranftl, A.~Dosovitskiy, J.~J. Chung, and R.~Siegwart, ``Learning to predict the wind for safe aerial vehicle planning,'' in \emph{2019 International Conference on Robotics and Automation (ICRA)}.\hskip 1em plus 0.5em minus 0.4em\relax IEEE, 2019, pp. 2311--2317.

\bibitem{luders2016wind}
B.~Luders, A.~Ellertson, J.~P. How, and I.~Sugel, ``Wind uncertainty modeling and robust trajectory planning for autonomous parafoils,'' \emph{Journal of Guidance, Control, and Dynamics}, vol.~39, no.~7, pp. 1614--1630, 2016.

\bibitem{wolek2015feasible}
A.~Wolek and C.~Woolsey, ``Feasible dubins paths in presence of unknown, unsteady velocity disturbances,'' \emph{Journal of Guidance, Control, and Dynamics}, vol.~38, no.~4, pp. 782--787, 2015.

\bibitem{shkel2001classification}
A.~M. Shkel and V.~Lumelsky, ``Classification of the dubins set,'' \emph{Robotics and Autonomous Systems}, vol.~34, no.~4, pp. 179--202, 2001.

\bibitem{lim2023circling}
J.~Lim, F.~Achermann, R.~B{\"a}hnemann, N.~Lawrance, and R.~Siegwart, ``Circling back: Dubins set classification revisited,'' in \emph{Workshop on Energy Efficient Aerial Robotic Systems, International Conference on Robotics and Automation 2023}, 2023.

\bibitem{swisstopo2023swissalti3d}
Swisstopo, ``Swissalti3d,'' 2023.

\bibitem{eu2019commission}
E.~Union, ``Commission implementing regulation (eu) 2019/947 of 24 may 2019 on the rules and procedures for the operation of unmanned aircraft,'' \emph{Official Journal of the European Union}, vol.~62, pp. 45--71, 2019.

\end{thebibliography}

\end{document}